\documentclass{llncs} 
\usepackage{amsmath}
\usepackage{amssymb}
\usepackage{graphicx}
\usepackage{color}

\begin{document}

\pagestyle{headings}

\title{\Large{Non-Gaussian Scale Space Filtering with 2$\times$2 Matrix of Linear Filters}}


\author{Toshiro Kubota}
\authorrunning{Kubota}
\institute{Department of Mathematics and Computer Science\\
Susquehanna University\\
Selinsgrove, PA 17840}

\maketitle

\begin{abstract}
Construction of a scale space with a convolution filter has
been studied extensively in the past. It has been proven that
the only convolution kernel that satisfies the scale space
requirements is a Gaussian type. In this paper, we consider a
matrix of convolution filters introduced in
\cite{Kubota:IJCV2009} as a building kernel for a scale space,
and shows that we can construct a non-Gaussian scale space with
a $2\times 2$ matrix of filters. The paper derives sufficient
conditions for the matrix of filters for being a scale space
kernel, and present some numerical demonstrations.
\end{abstract}

\section{Introduction}
Linear scale-space representations have been applied to many
signal and image processing
problems\cite{Witkin:ScaleSpaceA}\cite{Lindeberg:PAMI90}, in
which an optimum amount of smoothing cannot be determined in
advance. The linear scale space smoothing iteratively applies a
linear diffusion operator to the signal until an appropriate
amount of smoothing is introduced.

Recently, we proposed a new approach to signal smoothing
\cite{Kubota:IJCV2009}, which is linear, diffusion-like, but
possesses different frequency characteristics from the linear
diffusion operator; as the number of iteration increases, our
filter develops a sharper cut-off but retains the bandwidth
much longer than the linear diffusion. The filter was designed
from a geometrical perspective and called \textit{Elastic
Quadratic Wire} (EQW). We can consider EQW smoothing as
applications of linear filters (in particular circular
convolution filters) to the signal and its auxiliary extensions
in a computational structure similar to a linear transformation
by a $3\times 3$ matrix where each component is one of the
convolution filters.

Our goal is to understand the frequency characteristics of EQW
and derive general requirements on the filter coefficients to
meet the scale space
criteria\cite{Koenderink:ScaleSpace}\cite{Lindeberg:PAMI90}. It
has been shown that the only convolution kernel that satisfies
the scale space requirements is a Gaussian type for the
continuous time space
\cite{Babaud:ScaleSpace}\cite{Koenderink:ScaleSpace}\cite{Yuille:ScaleSpace}
and the modified Bessel functions of integer order for the
discrete time space. The latter approaches the Gaussian kernel
as the length of the filter increases. In this paper, instead
of considering a convolution filter, we consider a matrix of
convolution filters employed for the EQW smoothing. In
particular, instead of $3\times 3$ matrix of filters as in the
original EQW, we study a $2\times 2$ form. Although smaller in
size, the configuration retains some of intrinsic
characteristics of the original EQW and allows us to
characterize the frequency response algebraically. We will
extend the results in the future for larger and more general
configurations.

\section{Background}
We consider the following linear system.
\begin{equation}\label{eq:EQW}
\left(\begin{array}{c}
        \mathbf{x}_1^{[l+1]} \\
        \mathbf{x}_2^{[l+1]} \\
        \vdots \\
        \mathbf{x}_P^{[l+1]}
      \end{array}
\right)=(1-t)\left(\begin{array}{c}
        \mathbf{x}_1^{[l]} \\
        \mathbf{x}_2^{[l]} \\
        \vdots \\
        \mathbf{x}_P^{[l]}
                   \end{array}
\right) +
t\left(
        \begin{array}{c}
        \sum_{s=1}^P \mathbf{f}_{1s} \ast \mathbf{x}_s^{[l]}\\
        \sum_{s=1}^P \mathbf{f}_{2s} \ast \mathbf{x}_s^{[l]}\\
          \vdots\\
        \sum_{s=1}^P \mathbf{f}_{Ps} \ast \mathbf{x}_s^{[l]}\\
        \end{array}
      \right)
\end{equation}
where $P$ is the order of the system, $\mathbf{x}_s^{[\cdot]}$
($1\le s\le P$) are discrete signals of length $N$ in which the
number inside $[~]$ indicates the iteration number, $t\ge 0$ is
a scale parameter, $\mathbf{f}_{rs}$ ($1\le r,s\le P$) are
linear filters, and $\ast$ is a circular convolution operator.
The operator is applied iteratively, and (\ref{eq:EQW}) implies
that the outputs of $l$th iteration becomes the inputs to the
$l+1$st iteration. We call this computational structure a
\textit{matrix of filters}, as the filters can be arranged in a
$P\times P$ matrix form and the operation can be conveniently
viewed as a \textit{multiplication} (defined as in
(\ref{eq:EQW})) of the matrix with the input signals. Figure
\ref{fig:Schematic} shows two stages of a $P\times P$ matrix of
filters. Note that the computation at each stage is identical.

\begin{figure}
\centering
\includegraphics[width=5in]{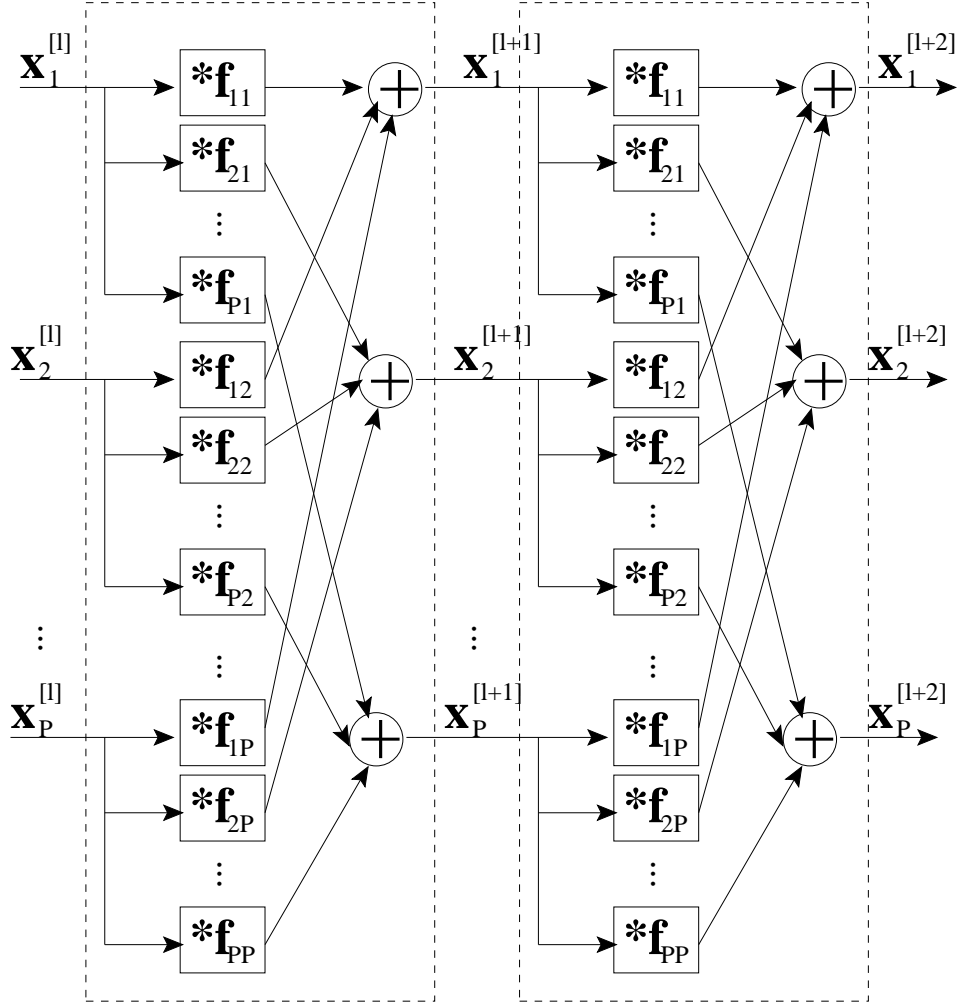}
\caption{A schematic of a $P\times P$ matrix of linear filters. A rectangle with
$*\mathbf{f}_{rs}$ is a linear filter whose impulse response is $\mathbf{f}_{rs}$. A circle
with $+$ receives $P$ inputs and adds them at every cycle.}
\label{fig:Schematic}
\end{figure}

\subsection{Equivalent filter} Let $\mathbf{M}_{rs}$ be a $N$
by $N$ circulant matrix that implements the linear filter of
$\mathbf{f}_{rs}$. Let
\begin{equation}
\mathbf{M}=
\left(
        \begin{array}{ccc}
          \mathbf{M}_{11} & \hdots & \mathbf{M}_{1P} \\
          \vdots & \ddots & \vdots\\
          \mathbf{M}_{P1} & \hdots & \mathbf{M}_{PP} \\
        \end{array}
      \right).
\end{equation}
Write the $l$th power of $\mathbf{M}$ as
\begin{equation}
\mathbf{M}^l=
\left(
        \begin{array}{ccc}
          \mathbf{M}_{11}^{[l]} & \hdots & \mathbf{M}_{1P}^{[l]} \\
          \vdots & \ddots & \vdots\\
          \mathbf{M}_{P1}^{[l]} & \hdots & \mathbf{M}_{PP}^{[l]} \\
        \end{array}
      \right).
\end{equation}
Then, the signal at $l$th iteration can be expressed in terms
of the initial signals
($\mathbf{x}_1^{[0]}\hdots\mathbf{x}_P^{[0]}$) by
\begin{equation}\label{eq:EQW2b}
\left(\begin{array}{c}
        \mathbf{x}_1^{[l]} \\
        \vdots \\
        \mathbf{x}_P^{[l]} \\
      \end{array}
\right)=
\left(
        \begin{array}{ccc}
          \mathbf{M}_{11}^{[l]} & \hdots & \mathbf{M}_{1P}^{[l]} \\
          \vdots & \ddots & \vdots\\
          \mathbf{M}_{P1}^{[l]} & \hdots & \mathbf{M}_{PP}^{[l]} \\
        \end{array}
      \right)
\left(\begin{array}{c}
        \mathbf{x}_1^{[0]} \\
        \vdots \\
        \mathbf{x}_P^{[0]} \\
      \end{array}
\right).
\end{equation}

Now we designate $\mathbf{x}_1$ as the primary signal and
$\mathbf{x}_s$ ($s>1$) as auxiliary ones. Initially, all
auxiliary signals are set to zero. Then, $\mathbf{x}_1^{[0]}$
and $\mathbf{M}_{11}^{[l]}$ determine $\mathbf{x}_1^{[l]}$, the
primary signal at the $l$th iteration. $\mathbf{M}_{11}^{[l]}$
is also circulant as circulant matrices are closed under
addition and multiplication. Therefore, it implements circular
convolution of a filter denoted as $\mathbf{f}_{11}^{[l]}$,
which we call an \textit{equivalent filter} at $l$th iteration.
The equivalent filter transforms $\mathbf{x}_1^{[0]}$ to
$\mathbf{x}_1^{[l]}$.

\subsection{Eigen-decomposition}
Define
\begin{equation}
\mathbf{B}_i=\left(\begin{array}{ccc}
                     \mathbf{f}_{11}(i) & \hdots & \mathbf{f}_{1P}(i) \\
                     \vdots & \ddots & \vdots \\
                     \mathbf{f}_{P1}(i) & \hdots & \mathbf{f}_{PP}(i)
                   \end{array}
                   \right)
\end{equation}
where $\mathbf{f}_{rs}(i)$ denotes the $i$th coefficient of
$\mathbf{f}_{rs}$ and
\begin{equation}
\mathbf{H}(\rho^k)=\sum_{i=0}^{N-1}\mathbf{B}_i \rho^{ik}
\end{equation}
where $\rho=e^{j2\pi/N}$. Then, the eigenvalues of
$\mathbf{H}\left(\rho^k\right)$ ($0\le k\le N-1$) are
eigenvalues of $\mathbf{M}$. For each $k$, there are P
eigenvalues. Thus, there are total of $PN$ eigenvalues for
$\mathbf{M}$ with possible repetition. Let $\lambda_i^k$
($i\in\left\{1,\hdots P\right\}$) be an eigenvalue of
$\mathbf{H}(\rho^k)$ and $\mathbf{v}_i^k$ be the corresponding
eigenvector. Let
$\mathbf{r}^k=\left[\rho^0~\rho^k\cdots\rho^{k(N-1)}\right]^T$
(with the superscript $T$ denotes transposition). Then
$\mathbf{v}_i^k\bigotimes \mathbf{r}(\rho^k)$ where
$\bigotimes$ denotes the Kronecker product is an eigenvector of
$\mathbf{M}$.

Let
$\mathbf{V}^k=\left[\mathbf{v}_1^k\hdots\mathbf{v}_P^k\right]$,
$\mathbf{g}^k$ be the first row of $\mathbf{V}^k$, and
$\tilde{\mathbf{g}}^k$ be the first column of
$\left(\mathbf{V}^k\right)^{-1}$. Thus,
\begin{equation}\label{eq:mixing}
\mathbf{g}^k\cdot\tilde{\mathbf{g}}^k=1.
\end{equation}

Let $\mathbf{D}_i$ ($i\in\left\{1\hdots P\right\}$) be an $N$
by $N$ diagonal matrix where $k$th diagonal component is
$\lambda_i^k$ and $\mathbf{G}_i$ be another $N$ by $N$ diagonal
matrix where $k$th diagonal component is
$\mu_i^k=\mathbf{g}^k(i)\tilde{\mathbf{g}}^k(i)$. Note that
$\mathbf{g}^k(i)$ denotes the $i$th component of
$\mathbf{g}^k$. We call $\mu_i^k$ \textit{mixing coefficients},
and
\begin{equation}
\sum_i\mu_i^k=1,
\end{equation}
for all $0\le k\le N-1$ due to (\ref{eq:mixing}).

Then, $\mathbf{M}_{11}^{[l]}$ can be decomposed by
\begin{equation}
\mathbf{M}_{11}^{[l]}=\mathbf{W}\sum_{i=1}^P\mathbf{D}_i^l\mathbf{G}_i\mathbf{W}^{T}
\end{equation}
where $\mathbf{W}$ is the $N$ by $N$ discrete Fourier
transformation matrix and $W_{rs}=\rho^{rs}=e^{j2\pi rs/N}$.
Hence, $\sum_{i=1}^P\mathbf{D}_i^l\mathbf{G}_i$ gives the
frequency response of the equivalent filter at $l$th iteration.

\subsection{Scale-Space Filters} We are interested in
incremental smoothing of signals with small size filters. In
literature, the approach is often referred to as \textit{scale
space filtering} and plays an important role in many signal and
image processing applications. To maximize the control of the
smoothing, we limit the number of non-zero filter coefficients
in $\mathbf{f}_{rs}$ to three, the smallest symmetric filter
size that allows construction of scale space. Thus, we assume
$\mathbf{B}_i=\mathbf{0}$ for $i\notin\left\{-1, 0,1\right\}$.
In this paper, we consider requirements on $\mathbf{f}_{rs}$
for a $2\times 2$ matrix of filters so that its equivalent
filter satisfies conditions for a scale space filter.

We impose the following conditions on equivalent filters.
\begin{enumerate}
\item Real frequency response: $\mathbf{f}_{11}^{[l]}$ is
    symmetric. In other words, the Fourier transform of the
    equivalent filter is real.
\item Positive response: Its frequency response is
    non-negative at every frequency component.
\item Unimodal response: Its frequency response is unimodal
    with the peak at the frequency 0.
\item Consistent reduction response: Each frequency
    component is non-increasing with respect to the
    iteration number.
\item Normalized response: The DC component of the Fourier
    transform is 1.
\item Equivalence to linear diffusion: It can be reduced to
    a common linear diffusion operator when the auxiliary
    signals are fixed at zero.
\end{enumerate}
The first requirement prevents any phase distortion after the
smoothing. The second through fourth requirements prevent any
new local minimum or local maximum from forming as the result
of smoothing, and are considered essential for scale space
representations. Note that the fourth requirement is trivially
satisfied for a conventional scale space filtering when the
second requirement is satisfied. That is not the case for the
matrix of filter based construction. The fifth requirement
preserves the mean value of the signal. The equivalent to
linear diffusion requirement states that the system without
contributions from the auxiliary signal will result in linear
diffusion. An iteration formula for the linear diffusion is
\cite{Lindeberg:PAMI90}
\begin{equation}\label{eq:LinearDiffusion}
\mathbf{x}_1^{k+1}=\left(1-2t\right)\mathbf{x}_1^k+
t\left(z\mathbf{x}_1^k+z^{-1}\mathbf{x}_1^k\right).
\end{equation}
where $z$ and $z^{-1}$ shift signals by one to the left and
right, respectively, and $t\in\left[0, 1/4\right]$.

\section{Formulation}\label{formulation}
We introduce notations specific to the case of $2\times 2$
matrix of filters. To reduce the amount of arabic subscripts,
we use $\mathbf{x}$ instead of $\mathbf{x}_1$ as the primary
signal and $\mathbf{a}$ for the sole auxiliary signal, and
write the processing of the $2\times 2$ matrix of filters as
\begin{equation}\label{eq:EQW2by2}
\left(\begin{array}{c}
        \mathbf{x}^{[l+1]} \\
        \mathbf{a}^{[l+1]}
      \end{array}
\right)=\left(\begin{array}{c}
        \mathbf{x}^{[l]} \\
        \mathbf{a}^{[l]}
                   \end{array}
\right) +
t\left(
        \begin{array}{c}
        \mathbf{f}_{xx} \ast \mathbf{x}^{[l]} + \mathbf{f}_{xa} \ast \mathbf{a}^{[l]}\\
        \mathbf{f}_{xa} \ast \mathbf{x}^{[l]} + \mathbf{f}_{aa} \ast \mathbf{a}^{[l]}\\
        \end{array}
      \right).
\end{equation}
Each convolution filters have at most three non-zero
coefficients. Thus, we write them
$\mathbf{f}_{\ast\ast}=[\alpha_{\ast\ast}~
\beta_{\ast\ast}~\gamma_{\ast\ast}]$ with one tap delay so that
$\beta_{\ast\ast}$ is the center of the filter. (Replace
$\ast\ast$ with $xx$, $xa$, $ax$, or $aa$.)

Note that $k\in [0, N-1]$ where $N$ is the length of the input
signal. Thus, as $N\rightarrow\infty$, $\rho^k$ covers all
roots of unity. Since we want to derive design requirements for
any signal length, we treat all quantities as functions of
$\rho=\left\{z\in\mathbb{C}| |z|=1\right\}$ or equivalently
$\theta=\angle\rho\in [-\pi, \pi]$. This allows us to
generalize our discussion and eliminate the superscript $k$
from expressions.

With these notations, eigenvalues of $\mathbf{B}$ are
\begin{align}\label{eq:eigenvalues}
\lambda_1(\rho)=\frac{\sigma_{xx}(\rho)+\sigma_{aa}(\rho)-\sqrt{\Delta(\rho)}}{2}\\
\lambda_2(\rho)=\frac{\sigma_{xx}(\rho)+\sigma_{aa}(\rho)+\sqrt{\Delta(\rho)}}{2},\\
\end{align}
and mixing coefficients are
\begin{align}
\mu_1(\rho)=-\frac{\sigma_{xx}(\rho)-\sigma_{aa}(\rho)}{2\sqrt{\Delta(\rho)}}+\frac{1}{2}\label{eq:mixingcoeff1}\\
\mu_2(\rho)=\frac{\sigma_{xx}(\rho)-\sigma_{aa}(\rho)}{2\sqrt{\Delta(\rho)}}+\frac{1}{2}\label{eq:mixingcoeff2}
\end{align}
where
\begin{align}
\sigma_{\ast\ast}(\rho)=\delta_{\ast\ast}(\rho)+t\left(\alpha_{\ast\ast}\rho^{-1}+\beta_{\ast\ast} + \gamma_{\ast\ast}\rho\right)\\
\Delta(\rho) = \left(\sigma_{xx}(\rho)-\sigma_{aa}(\rho)\right)^2+4\sigma_{ax}(\rho)\sigma_{xa}(\rho).
\end{align}
with $\delta_{rs}$ being the Kronecker's delta function (1 when
$r=s$ and 0 otherwise). Note that in (\ref{eq:mixingcoeff1})
and (\ref{eq:mixingcoeff2}), we are assuming that the
eigenvalues are distinct (or $\Delta\ne 0$). When
$\lambda_1=\lambda_2$, the mixing coefficients are arbitrary,
and we set $\mu_1=\mu_2=1/2$.

Note that the above expressions are all functions of $\rho$ (or
equivalently $\theta$). However, for brevity, we omit the
variable in their expressions unless we are evaluating them at
a particular $\rho$.

The frequency response of the equivalent filter at $l$th
iteration is
\begin{align}\label{eq:FreqResp}
F^l=D_1^lG_1+D_2^lG_2=\notag\\
\mu_1
\left(\frac{\sigma_{xx}+\sigma_{aa}-\sqrt{\Delta}}{2}\right)^l +
\mu_2
\left(\frac{\sigma_{xx}+\sigma_{aa}+\sqrt{\Delta}}{2}\right)^l.
\end{align}

\section{Filter Requirements}
In this section, we derive a sufficient condition for a
$2\times 2$ matrix of filters to satisfy the scale space
requirements. We first use necessary conditions to simplify the
formulae of the individual filters. We then use the simplified
formulae to derive sufficient conditions on the design
parameters.

\begin{theorem}\label{theorem1}
Necessary conditions for a $2\times 2$ matrix of filters being
a scale space filter are
\begin{align}
\sigma_{xx}=\left(1-2t\right)+2t\cos\theta\label{eq:sxx}\\
\sigma_{aa}=\left(1-2bt\right)+2tc\cos\theta\label{eq:saa}\\
\sigma_{xa}\sigma_{ax}=-4t^2d\sin^2\theta\label{eq:sxaax}
\end{align}
where $b,c,d\in \mathbb{R}$.
\end{theorem}
\begin{proof}
$F^l$ at the first three iterations ($l=1,2,3$) are
\begin{equation}
F^1=\sigma_{xx},
\end{equation}
\begin{equation}
F^2=\left(\sigma_{xx}\right)^2+\sigma_{xa}\sigma_{ax},
\end{equation}
\begin{equation}
F^3=\left(\sigma_{xx}\right)^3+2\sigma_{xx}\sigma_{xa}\sigma_{ax}+
\sigma_{aa}\sigma_{xa}\sigma_{ax}.
\end{equation}
Since they have to be real according to the real frequency
response requirement, $\sigma_{xx}$, $\sigma_{ax}\sigma_{xa}$
and $\sigma_{aa}$ have to be real. $\sigma_{xx}$ is real for
any $k$ if and only if $\alpha_{xx}=\gamma_{xx}$ or
$\mathbf{f}_{xx}$ is symmetric. $\sigma_{aa}$ is real for any
$k$ if and only if $\alpha_{aa}=\gamma_{aa}$ or
$\mathbf{f}_{aa}$ is symmetric. $\sigma_{ax}\sigma_{xa}$ is
real for any $k$ if and only if either both $\sigma_{ax}$ and
$\sigma_{xa}$ are real or both are imaginary. Thus,
$\mathbf{f}_{ax}$ and $\mathbf{f}_{xa}$ are either both
symmetric or anti-symmetric. Below, we show that either way, we
will get the same condition for $\sigma_{ax}\sigma_{xa}$. Now,
let's introduce constraints derived from other requirements.

Due to the equivalence to linear diffusion requirement, we set
$\alpha_{xx}=\gamma_{xx}=1$ and $\beta_{xx}=-2$, which gives
\begin{equation}\label{eq:xx_LD}
\sigma_{xx}=(1-2t)+2t\cos\theta.
\end{equation}

For the normalized response requirement, the frequency response
at $\theta=0$ has to be 1 for all $l$. Thus,
$F^1(\theta=0)=F^2(\theta=0)=1$ has to be true. This gives
$\sigma_{xx}(\theta=0)=1$ and
$\sigma_{xa}(\theta=0)\sigma_{ax}(\theta=0)=0$. The former is
satisfied with (\ref{eq:xx_LD}). The latter gives
\begin{equation}
\left(\alpha_{xa}+\beta_{xa}+\gamma_{xa}\right)\left(\alpha_{ax}+\beta_{ax}+\gamma_{ax}\right)=0.\label{eq:xaax_NR}
\end{equation}

Due to the constant reduction and the positivity requirements,
$0\le F^2(\theta)\le F^1(\theta)$ or
$0\le\sigma_{xx}^2+\sigma_{xa}\sigma_{ax}\le\sigma_{xx}$ for
any $t\in[0,~1/4]$. Consider $t=1/4$. Then,
$\sigma_{xx}(\theta=\pi)=0$. Hence, we have
\begin{equation}\label{eq:xaax_CRPOS}
\sigma_{xa}(\pi)\sigma_{ax}(\pi)=\left(\beta_{xa}-\alpha_{xa}-\gamma_{xa}\right)
\left(\beta_{ax}-\alpha_{ax}-\gamma_{ax}\right)=0.
\end{equation}

Now, we investigate if $\mathbf{f}_{xa}$ and $\mathbf{f}_{ax}$
should be both symmetric or anti-symmetric. First, Assume both
are symmetric. Thus, $\gamma_{xa}=\alpha_{xa}$ and
$\gamma_{ax}=\alpha_{ax}$. Because of the symmetry of
(\ref{eq:xaax_NR}) and (\ref{eq:xaax_CRPOS}), we consider
$\left(\alpha_{ax}+\beta_{ax}+\gamma_{ax}\right)=0$ and
$\left(\beta_{xa}-\alpha_{xa}-\gamma_{xa}\right)=0$, and the
derivation for the other case will be the same. We have
$\alpha_{ax}=\gamma_{ax}=-\beta_{ax}/2$ and
$\alpha_{xa}=\gamma_{xa}=\beta_{xa}/2$, and obtain
\begin{equation}
\sigma_{xa}\sigma_{ax}=t^2\beta_{xa}\beta_{ax}\left(1-\cos^2\theta\right)=t^2\beta_{xa}\beta_{ax}\sin^2\theta.\\
\end{equation}

Next, assume that both $\mathbf{f}_{xa}$ and $\mathbf{f}_{ax}$
are anti-symmetric. Thus, $\gamma_{xa}=-\alpha_{xa}$,
$\gamma_{ax}=-\alpha_{ax}$, and $\beta_{xa}=\beta_{ax}=0$. In
this case, both (\ref{eq:xaax_NR}) and (\ref{eq:xaax_CRPOS})
are satisfied. Furthermore,
\begin{equation}
\sigma_{xa}\sigma_{ax}=-4t^2\alpha_{xa}\alpha_{ax}\sin^2\theta.\\
\end{equation}

Therefore, given symmetric version of $\mathbf{f}_{xa}$ and
$\mathbf{f}_{ax}$, we can always find anti-symmetric
counterpart that provides the same expression for
$\sigma_{xa}\sigma_{ax}$, and vice verse.

Let $\beta_{aa}=b$ and $\alpha_{aa}=\gamma_{aa}=c$. Then, we
have $\sigma_{aa}=\left(1-2bt\right)+2ct\cos\theta$. Let both
$\mathbf{f}_{xa}$ and $\mathbf{f}_{ax}$ are anti-symmetric, and
$\alpha_{ax}\alpha_{xa}=d$. Then, we obtain the necessary
condition. $\Box$
\end{proof}

We have found that all scale space filters can be expressed by
(\ref{eq:sxx}), (\ref{eq:saa}), and (\ref{eq:sxaax}). There are
three design parameters: $b$, $c$, and $d$. The next theorem
provides requirements on these parameters.

\begin{theorem}\label{theorem2}
Sufficient conditions for a $2\times 2$ matrix of filters being
a scale space filter are that they are in the forms of
(\ref{eq:sxx})-(\ref{eq:sxaax}) and
\begin{align}
0\leq b+c\leq 2,\\
-2d\leq b-c\geq -2d,\\
c(b-c)\geq 2d,\\
c^2-(2-b)c+2d\leq 0
\end{align}
\end{theorem}
\begin{proof}
We first derive constraints on the three parameters for each
scale space requirement.

\textbf{Real frequency response requirement:}

Since
$\sigma_{xx}$, $\sigma_{aa}$ and $\sigma_{xa}\sigma_{ax}$ are
all real and consequently $\Delta$, $\mu_1$, $\mu_2$,
$\lambda_1$, and $\lambda_2$ are all real, the real frequency
response is satisfied.

\textbf{Equivalence to linear diffusion requirement:}
The equivalence to linear diffusion requirement is satisfied
with (\ref{eq:sxx}).

\textbf{Normalized response requirement:}
At $\theta=0$, $\sigma_{xa}\sigma_{ax}=0$ and
$\Delta=\left(\sigma_{xx}^0-\sigma_{aa}^0\right)^2$. If
$\sigma_{xx}^0\geq \sigma_{aa}^0$, $\lambda_1=0$,
$\lambda_2=1$, $\mu_1=0$, and $\mu_2=1$. Thus,
$F^l\left(\theta=0\right)=1$ for all $l\geq 1$. If
$\sigma_{xx}^0< \sigma_{aa}^0$, $\lambda_1=1$, $\lambda_2=0$,
$\mu_1=1$, and $\mu_2=0$. Again, $F^l\left(\theta=0\right)=1$
for all $l\geq 1$. Therefore, the normalized response
requirement is satisfied.

\textbf{Positive response and consistent reduction requirements:}
For convenience, we treat the two requirements together. Let
\begin{align}
\phi=\sigma_{xx}+\sigma_{aa},\label{eq:phi}\\
\psi=\sigma_{xx}-\sigma_{aa}.\label{eq:psi}
\end{align}
Then,
\begin{equation}\label{eq:FreqRespB}
F^l=
\frac{\sqrt{\Delta}-\psi}{2\sqrt{\Delta}}
\left(\frac{\phi-\sqrt{\Delta}}{2}\right)^l +
\frac{\sqrt{\Delta}+\psi}{2\sqrt{\Delta}}
\left(\frac{\phi+\sqrt{\Delta}}{2}\right)^l.
\end{equation}

The positive response requirement is satisfied if
\begin{align*}
\psi\leq \sqrt{\Delta}\\
\phi\geq\sqrt{\Delta},
\end{align*}
and the consistent reduction requirement is satisfied if
\begin{align*}
\psi\leq \sqrt{\Delta}\\
\phi+\sqrt{\Delta}\leq 2.
\end{align*}
Thus, both conditions are satisfied if
\begin{equation}
\psi\leq \sqrt{\Delta}\label{eq:th2_req1}
\end{equation}
\begin{equation}
\sqrt{\Delta}\leq \phi\leq 2-\sqrt{\Delta}\label{eq:th2_req2}.
\end{equation}
Since $\Delta=\left(\psi\right)^2-d\sin^2\theta$,
(\ref{eq:th2_req1}) is satisfied for all $\theta$ if
\begin{equation}\label{eq:req1}
d\leq 0.
\end{equation}
With $\phi\leq 2-\sqrt{\Delta}$, we have
\begin{equation}
\Delta\leq\left(2-\phi\right)^2, \phi\leq 2.
\end{equation}
Now,
\begin{equation}\label{eq:ineq1}
(2-\phi)^2-\Delta=16t^2\left(\left(c-d\right)w-\left(b+d\right)\right)\left(w-1\right)\geq 0
\end{equation}
where $w=\cos\theta$. Since $-1\le w\le 1$, $w-1\leq 0$. Thus,
(\ref{eq:ineq1}) is satisfied if
\begin{equation}
\left(c-d\right)w-\left(b+d\right)\leq 0.
\end{equation}
The above inequality holds if and only if it holds at $w=1$ and
$w=-1$. Hence,
\begin{equation}
b-c\geq -2d\label{eq:req2}
\end{equation}
\begin{equation}
b+c\geq 0\label{eq:req3}.
\end{equation}
With $\sqrt{\Delta}\leq \phi$, we have
$\sigma_{xx}\sigma_{aa}\geq\sigma_{xa}\sigma_{ax}$ or
\begin{equation}\label{eq:PositiveReq3}
(1-2t+2tw)(1-2tb+2tcw)\geq -4dt^2(1-w^2).
\end{equation}
Define
\begin{align*}
\eta(t,w)=(1-2t+2tw)(1-2tb+2tcw)+4dt^2(1-w^2)\\
=4(1-w)\left(b+d-(c-d)w\right)t^2-2\left(1+b-(1+c)w\right)t+1.
\end{align*}
To show (\ref{eq:PositiveReq3}), we need to show $\eta(t,w)\geq
0$.

Before deriving sufficient conditions for $\eta(t,w)\geq 0$, we
find two necessary conditions for (\ref{eq:PositiveReq3}). We
then show that the two conditions together with
(\ref{eq:req1}), (\ref{eq:req2}), and (\ref{eq:req3}) form
sufficient conditions for $\eta(t,w)\geq 0$.

The right hand side of (\ref{eq:PositiveReq3})is non-negative
with(37) and $1-2t+2tw\geq 0$ for $0\leq t\leq 1/4$ and $-1\leq
w\leq 1$. Thus, it is necessary that $1-2tb+2tcw\geq 0$. By
setting $t=1/4$ and $w=1$, we have
\begin{equation}
b-c\leq 2\label{eq:req4}.
\end{equation}
By setting $t=1/4$ and $w=-1$, we have
\begin{equation}
b+c\leq 2\label{eq:req5}.
\end{equation}

When $(1-w)\left(b+d-(c-d)w\right)=0$, $\eta(t,w)\geq 0$
because
\begin{equation*}
\eta(t,1)=1-2(b-c)t\geq 1-(b-c)/2\geq 0
\end{equation*}
(using $t\leq 1/4$ and $b-c\leq 2$) and
\begin{equation*}
\eta(t,-1)=1-2\left(2+(b+c)\right)t=1-4t\geq 0
\end{equation*}
(using $b+c=0$ derived from $b+d-(c-d)(-1)=0$ and $t\leq 1/4$),
and $\eta(t,w)$ in this case is a linear function of $w$.

With $(1-w)\left(b+d-(c-d)w\right)\ne 0$, we have
\begin{align*}
\eta(t,w)=4(1-w)\left(b+d-(c-d)w\right)\left(t-\frac{\left(1+b-(1+c)w\right)}{4(1-w)\left(b+d-(c-d)w\right)}\right)^2\\
-\frac{\left(1+b-(1+c)w\right)^2}{4(1-w)\left(b+d-(c-d)w\right)}+1.
\end{align*}
This is convex ($4(1-w)\left(b+d-(c-d)w\right)\geq 0$) parabola
of $t$ with its center located at the positive side
($1+b-(1+c)w\geq 0$). A necessary and sufficient condition for
the above inequality is
\begin{equation}
\eta(0,w)\geq 0,
\end{equation}
\begin{equation}
\eta(1/4,w)\geq 0,
\end{equation}
and
\begin{equation}\label{eq:misc1}
1-\frac{\left(1+b-(1+c)w\right)^2}{4(1-w)\left(b+d-(c-d)w\right)}\geq 0
\end{equation}
provided that
\begin{equation}\label{eq:misc2a}
\frac{1+b-(1+c)w}{4(1-w)\left(b+d-(c-d)w\right)}\leq 1/4
\end{equation}
or
\begin{equation}\label{eq:misc2}
0\leq 1+b-(1+c)w\leq (1-w)\left(b+d-(c-d)w\right).
\end{equation}
The first condition is trivially satisfied. The second
condition is satisfied since
\begin{equation}
\eta(1/4,w)=\frac{1}{4}(1-w)\left(2-(b+d)+(c+d)w\right)
\end{equation}
and $1-w\geq 0$ and $2-(b+d)+(c+d)w\geq 0$. (Note that at
$w=1$, we have $2-(b-c)\geq 0$ because $b-c\leq 2$, and at
$w=-1$, we have $2-2d-(b+c)\geq 0$ because $b+c\leq 2\leq
2-2d$.)

For the third condition, (\ref{eq:misc1}) with (\ref{eq:misc2})
gives
\begin{equation}
1-\frac{\left(1+b-(1+c)w\right)^2}{4(1-w)\left(b+d-(c-d)w\right)}\geq 1-\frac{1}{4}(1-w)\left(b+d-(c-d)w\right)
\end{equation}
Thus, the third condition is satisfied if the right hand side
of the above is non-negative. In other words, we need to show
\begin{equation}
\zeta(w)=4-(1-w)\left(b+d-(c-d)w\right)=(c-d)w^2+(b-1)w+(1-d)
\end{equation}
is non-negative. Indeed, $\zeta(-1)=2-2d-(b-c)\geq 0$ and
$\zeta(1)=b+c-2d\geq 0$. Therefore, if $c-d\leq 0$, then
$\zeta(w)\geq 0$ in $-1\leq w\leq 1$. If $c-d> 0$, then
\begin{equation}
\zeta(w)=(c-d)\left(w+\frac{b+c}{2(c-d)}\right)^2-\frac{(b+c)^2}{4(c-d)}+4-(b+d)
\end{equation}
and we need to show that
\begin{equation}
-\frac{(b+c)^2}{4(c-d)}+4-(b+d)\geq 0
\end{equation}
provided $(b+c)/(c-d)<2$. This is indeed the case, since
\begin{equation}
-\frac{(b+c)^2}{4(c-d)}+4-(b+d)\geq 4-(c-d)-(b+d)=4-(b+c)\geq 2.
\end{equation}

\textbf{Unimodality Requirement:}
Note that
\begin{equation}
\frac{\partial F^l(\theta)}{\partial\theta}=-\xi\left(\cos\theta, t\right) \sin\theta,
\end{equation}
where
\begin{align*}
\xi(w,t)=\left(\left(\frac{\phi+\sqrt\Delta}{2}\right)^l+\left(\frac{\phi-\sqrt\Delta}{2}\right)^l\right)
\left(\frac{2\Delta\xi_\psi-\psi\xi_\Delta}{4\Delta\sqrt\Delta}\right)+\notag\\
l\left(\frac{-\psi}{2\sqrt\Delta}+\frac{1}{2}\right)\left(\frac{\phi-\sqrt\Delta}{2}\right)^{l-1}
\left(\frac{2\sqrt\Delta\xi_\phi-\xi_\Delta}{4\sqrt\Delta}\right)+\notag\\
l\left(\frac{-\psi}{2\sqrt\Delta}+\frac{1}{2}\right)\left(\frac{\phi+\sqrt\Delta}{2}\right)^{l-1}
\left(\frac{2\sqrt\Delta\xi_\phi+\xi_\Delta}{4\sqrt\Delta}\right),
\end{align*}
with
\begin{equation}
\xi_\phi(t)=2t(1+c),
\end{equation}
\begin{equation}
\xi_\psi(t)=2t(1-c),
\end{equation}
\begin{equation}
\xi_\Delta(w,t)=8t^2\left((1-c)\left(-(1-b)+(1-c)w\right)+4dw\right).
\end{equation}
Note that
$\partial\phi(\theta)/\theta=-\xi_\phi(t)\sin(\theta)$,
$\partial\psi(\theta)/\theta=-\xi_\psi(t)\sin(\theta)$, and
$\partial\Delta(\theta)/\theta=\xi_\Delta(t)\sin(\theta)$.

The unimodality requirement is satisfied if $\xi(w,t)\geq 0$ in
$-1\leq w\leq 1$ and $0\leq t\leq 1/4$. Observe that
$\xi(w,t)\geq 0$ if
\begin{align}
2\Delta\xi_\psi-\psi\xi_\Delta\geq 0\\
2\sqrt\Delta\xi_\phi-\xi_\Delta\geq 0\label{eq:misc3}\\
2\sqrt\Delta\xi_\phi+\xi_\Delta\geq 0\label{eq:misc4},
\end{align}
provided that the positivity requirement is satisfied.

Now,
\begin{equation}
2\Delta\xi_\psi-\psi\xi_\Delta =-64dt^3\left((1-c)+(1-b)w\right)
\end{equation}
is a linear function of $w$. At $w=1$,
\begin{equation}
2\Delta\xi_\psi-\psi\xi_\Delta =-64dt^3\left(2-(b+c)\right)\geq 0
\end{equation}
by (\ref{eq:req1}) and (\ref{eq:req5}). At $w=-1$,
\begin{equation}
2\Delta\xi_\psi-\psi\xi_\Delta =-64dt^3\left(b-c\right)\geq 0
\end{equation}
by (\ref{eq:req1}) and (\ref{eq:req2}). Hence
$2\Delta\xi_\psi-\psi\xi_\Delta\geq 0$.

(\ref{eq:misc3}) and (\ref{eq:misc4}) can be combined into
\begin{align}
\xi_\phi\geq 0\label{eq:misc5}\\
4\Delta\xi_\phi^2-\xi_\Delta^2\geq 0\label{eq:misc6}.
\end{align}
(\ref{eq:misc5}) is satisfied if and only if
\begin{equation}
\label{eq:req6}
c\geq -1.
\end{equation}

Let
\begin{equation}
\vartheta(w)=4\Delta\xi_\phi^2-\xi_\Delta^2.
\end{equation}
We first evaluate necessary conditions $\vartheta(1)\geq 0$ and
$\vartheta(-1)\geq 0$, then show that the necessary conditions
are also sufficient for (\ref{eq:misc6}).

\begin{equation}
\vartheta(1)=4(b-c+2d)\left(bc-2d-c^2\right)\geq 0.
\end{equation}
Since $b-c+2d\geq 0$ according to (\ref{eq:req2}),
$bc-2d-c^2\geq 0$. Hence, we have
\begin{equation}
\label{eq:req7}
bc-2d-c^2\geq 0.
\end{equation}
\begin{equation}
\vartheta(-1)=4(b+c-2d-2)\left(bc+2d+c^2-2c\right)\geq 0
\end{equation}
Thus, either $b+c-2d-2\geq 0$ and $bc+2d+c^2-2c\geq 0$ or
$b+c-2d-2\leq 0$ and $bc+2d+c^2-2c\leq 0$. However,
\begin{equation}
\left(b+c-2d-2\right)+\left(bc+2d+c^2-2d\right)=\left(b+c-2\right)\left(c+1\right)\leq 0
\end{equation}
according to (\ref{eq:req5}) and (\ref{eq:req6}). Therefore, we
have
\begin{equation}
\label{eq:req8}
b+c-2d-2\leq 0
\end{equation}
\begin{equation}
\label{eq:req9}
bc+2d+c^2-2c\leq 0.
\end{equation}

Now, we show that these conditions are sufficient for
$\vartheta(w)\geq 0$ in $-1\leq w\leq 1$. When
$\left(1-c\right)^2+4d=0$, $\vartheta(w)$ reduces to a linear
expression of $w$, and $\vartheta(1)\geq 0$ and
$\vartheta(-1)\geq 0$ are sufficient for $\vartheta(w)\geq 0$.
When $\left(1-c\right)^2+4d\ne 0$,
\begin{equation}
\vartheta(w)=\alpha\left(w-w_c\right)^2+\beta
\end{equation}
where
\begin{equation}
\alpha=c(2+c)\left(\left(1-c\right)^2+4d\right),
\end{equation}
\begin{equation}
w_c=\frac{(1-b)(1-c)}{\left(1-c\right)^2+4d},
\end{equation}
\begin{equation}
\beta=4d(1+c)^2\left(\frac{(1-b)^2}{(1-c)^2+4d}-1\right).
\end{equation}
\end{proof}
When $(1-c)^2+4d<0$, then $\vartheta(w)$ is concave, and
$\vartheta(1)\geq 0$ and $\vartheta(-1)\geq 0$ are sufficient
for $\vartheta(w)\geq 0$. When $(1-c)^2+4d>0$, then
$\vartheta(w)$ is convex and we consider three cases dependent
on $w_c$: $w_c\leq -1$, $w_c\geq 1$, and $-1<w_c<1$. When
$w_c\leq -1$, $\vartheta(-1)\geq 0$ is sufficient for
$\vartheta(w)\geq 0$. When $w_c\geq 1$, $\vartheta(1)\geq 0$ is
sufficient for $\vartheta(w)\geq 0$. For $-1\leq w_c\leq 1$, we
need to show $\vartheta(w_c)=\beta\geq 0$ or equivalently
\begin{equation}
\frac{(1-b)^2}{(1-c)^2+4d}\leq 1
\end{equation}
because $d\leq 0$.

Since $w_c^2<1$ and $(1-c)^2>-4d\geq 0$, we have
\begin{equation}
(1-b)^2<\frac{\left((1-c)^2+4d\right)^2}{\left(1-c\right)^2}.
\end{equation}
Thus,
\begin{equation}
\frac{(1-b)^2}{(1-c)^2+4d}<\frac{(1-c)^2+4d}{\left(1-c\right)^2}\leq 1
\end{equation}
where we used $d\leq 0$ for the second inequality. Hence
$\vartheta(w_c)\geq 0$.

Therefore, we have the following requirements.
\begin{align}
d\leq 0\label{eq:tobederived5}\\
0\leq b+c\leq 2+2d\label{eq:tobederived1}\\
-2d\leq b-c \leq 2\label{eq:tobederived6}\\
bc-c^2-2d\geq 0\label{eq:tobederived4}\\
bc+2d+c^2-2c\leq 0\label{eq:tobederived3}
\end{align}
$\Box$

Note that
\begin{align}
-d\leq b\leq 2+d,\label{eq:intervalB}\\
-1\leq c\leq 1+2d\label{eq:intervalC}.
\end{align}

\section{Numerical Experiments}
\label{experiments}
When $d=0$, $\sigma_{xa}\sigma_{ax}=0$. In
this case, according to (\ref{eq:FreqResp}),
\begin{equation}
F^l=\sigma_{xx}^l,
\end{equation}
and the filter reduces to the linear diffusion type.
Historically, the resulting scale space is called
\textit{Gaussian}. In this section, we compute frequency
responses of filters at various settings, and compare them to
the Gaussian scale space.

Let's first look at two instances of the linear diffusion type,
which can illustrate how the mixing coefficients ($\mu_1$ and
$\mu_2=1-\mu_1$) and two eigenfunctions ($\lambda_1(\theta)$
and $\lambda_2(\theta)$) contribute to the overall frequency
response $F^l(\theta)$.

The frequency responses of a matrix of filters with $t=1/4$,
$b=c=1$ and $d=0$ at $l=1$, 50, 100, and 150 are shown in
Figure \ref{fig:NonGaussianFilter1}. Note that the scale
parameter $t$ contributes to the speed of the smoothing and
does not change the scale space. In the figure, (a) shows
$F^l$, the magnitude of the frequency responses, (b) shows
$\mu_2$, a mixing coefficient, (c) shows $\lambda_1^l$, and (d)
shows $\lambda_2^l$. This is a special case where
$\Delta(\theta)=0$ since
$\sigma_{xx}(\theta)=\sigma_{aa}(\theta)$. Thus,
$\lambda_1(\theta)=\lambda_2(\theta)=(1-2t)+2t\cos\theta$. With
two eigenfunctions being equal, the mixing coefficients are
arbitrary. As stated in Section \ref{formulation}, we set them
to $\mu_1=\mu_2=1/2$. However, regardless of the choice of the
mixing coefficients,
$F^l(\theta)=\left(\mu_1(\theta)+\mu_2(\theta)\right)
\sigma_{xx}^l(\theta)= \sigma_{xx}^l(\theta) =
\left((1-2t)+2t\cos\theta\right)^l$.

The frequency responses with $t=1/4$, $b=1$, $c=1/2$, and $d=0$
are shown in Figure \ref{fig:GaussianFilter2} with the same
arrangement with Figure \ref{fig:GaussianFilter1}. In this
case, $\Delta(\theta)\neq 0$, and the mixing coefficients are
uniquely determined. They are either 0 or 1, and switch the
value at the point where the sign of $\sigma_{xx}-\sigma_{aa}$
changes. Since $\sigma_{aa}=1-2t+t\cos\theta$ with $b=1$ and
$c=1/2$, the switch occurs at $\theta=\pi/2$. When
$\mu_k(\theta)=1$, $\lambda_k=\sigma_{xx}$, and when
$\mu_k(\theta)=0$, $\lambda_k=\sigma_{aa}$. Note that
$\lambda_k(\theta)$ does not contribute to $F^l(\theta)$ when
$\mu_k(\theta)=0$. Therefore, $F^l=\sigma_{xx}^l$, and is not
dependent on $\sigma_{aa}$.

With $d$ strictly negative, $b$ has to be positive according to
(\ref{eq:intervalB}). We are allowed to set $c=0$ according to
(\ref{eq:intervalC}), which we will do since the setting
results in a simpler form of the filter (a smaller number of
non-zero coeffcients). Then, the requirements given in Theorem
\ref{theorem2} reduces to
\begin{align}
d\leq 0,\\
-2d\leq b\leq 2+2d.
\end{align}
The smallest allowable $d$ is thus $-0.5$, which makes $b=1.0$.
We choose the smallest $d$ so that the resulting filter may
exhibit behavior that is more distinguishable from the linear
diffusion case than the one with $d\approx 0$.

The frequency responses of the above matrix of filters
($t=1/4$, $b=1$, $c=0$, and $d=-0.5$) at $l=1$, 50, 100, and
150 are observed and shown in Figure
\ref{fig:NonGaussianFilter1}. The arrangement of the plots in
the figure is the same as in Figure \ref{fig:GaussianFilter1}.
With $d\neq 0$, the mixing coefficients are no longer binary
and $\mu_2$ decreases from 1 to 0 as $\theta$ goes from 0 to
$\pi$, almost in a linear fashion. More specifically, with this
setting, we have
\begin{align}
\phi(\theta)=1+\frac{\cos\theta}{2}\\
\sqrt{\Delta}=\frac{\sqrt{6-2\cos2\theta}}{4}.
\end{align}
Note that
$\phi=\sigma_{xx}+\sigma_{aa}=(2-2((1+b)t)+2t(1+c)\cos\theta$.
Thus, $\phi$ (as well as $\psi=\sigma_{xx}-\sigma_{aa}$) is
always a $\cos\theta$ with scaling and offsetting. Since the
eigenfunctions are $\frac{\phi}{2}\pm \frac{\sqrt{\Delta}}{2}$,
$\sqrt{\Delta}$ is responsible for any characteristics of
$\lambda$s deviating from $\cos\theta$. Note that $\Delta$
consists of $\cos 2\theta$, which is a key in curving out a
frequency response profile that is different from the linear
diffusion case.

In Figure \ref{fig:EigenCloseup}, the eigenfunctions and their
constituents are shown. $\lambda_1$ and $\lambda_2$ are shown
with solid lines, $\phi/2$ is shown with a dotted line, and
$\sqrt{\Delta}/2$ is shown with dashed-dotted line.
$\sqrt{\Delta}/2$ provides deviation of $\lambda$s from
$\phi/2$, which is a scaled copy of $\cos\theta$ . To show the
degree of the deviation and how $\sqrt{\Delta}$ tunes
$\lambda$s, $\phi/2\pm 1/4$ are shown in Figure
\ref{fig:EigenCloseup} with dashed lines. These curves coincide
with $\lambda$s at $\theta=0$ and $\pi$ but deviate from them
elsewhere. By the comparisons, the slope of $\lambda_2$ (which
is above $\lambda_1$) is smaller around $\theta=0$ and larger
around $\theta=\pi$ than $\cos\theta$. On the other hand, the
slope of $\lambda_1$ is larger around $\theta=0$ and smaller
around $\theta=\pi$. By weighting more on $\lambda_2$ near
$\theta=0$ and more on $\lambda_1$ near $\theta=\pi$, we might
be able to bring a frequency response with a sharper cut-off
than $\cos\theta$, the Gaussian scale space case. However,
since $\mu_1\lambda_1+\mu_2\lambda_2=\sigma_{xx}$, this does
not happen at the first iteration. However, as the iteration
($l$) increases, $\lambda_2^l$ tends to hold its frequency
profile near $\theta=0$ better than the Gaussian counterpart,
thanks to its flatter profile around $\theta=0$. As a result,
the bandwidth of $F^l=\mu_1\lambda_1^l+\mu_2\lambda_2^l$ does
not diminish as quickly as that of $\sigma_{xx}^l$, which we
observe in Figure \ref{fig:NonGaussianFilter1}.

To implement the above filter, we can set the temporal filters
as
\begin{align}
\textbf{f}_{xx}=[t, ~1-2t, ~t]z^{-1}\\
\textbf{f}_{aa}=[1/2]\\
\textbf{f}_{ax}=[\sqrt{1/2}, ~0, ~-\sqrt{1/2}]z^{-1}\\
\textbf{f}_{xa}=[-\sqrt{1/2}, ~0, ~\sqrt{1/2}]z^{-1}.
\end{align}
Note that the requirements of $\textbf{f}_{ax}$ and
$\textbf{f}_{xa}$ are $\alpha_{ax}\alpha_{xa}=d=-1/2$,
$\beta_{ax}=\beta_{xa}=0$, $\gamma_{ax}=-\alpha_{ax}$, and
$\gamma_{xa}=-\alpha_{xa}$. Thus, the above setting is the most
balanced one, but just one of infinitely many. One may want to
set instead
\begin{align}
\textbf{f}_{ax}=[1, ~0, ~-1]z^{-1}\\
\textbf{f}_{xa}=[-1/2, ~0, ~1/2]z^{-1}.
\end{align}
so that they can be implemented more efficiently without
multipliers.


We claimed that the bandwidth of $F^l$ decreases more slowly
for the instance shown in Figure \ref{fig:NonGaussianFilter1}
than for the Gaussian case shown in Figure
\ref{fig:GaussianFilter1}. To reveal that scale spaces resulted
from these two filters are indeed different and not an
superficial one due simply to the speed of the smoothing, the
frequency responses of two filters from each respective
configurations are shown in Figure \ref{fig:FilterComparison}.
The response with $b=1$, $c=d=0$ at 5th iteration is shown in
dashed while the response with $b=1$, $c=0$ and $d=-0.5$ at
35th iteration is shown in solid. Note that due to the constant
reduction requirement of scale spaces, frequency responses of a
linear diffusion filter at two iteration points cannot
intersect each other. Thus, the filter response shown in solid
cannot be produced by the linear diffusion kernel, and the
resulting scale space is different from the Gaussian one.

\begin{figure}
\centering
\includegraphics[width=4in]{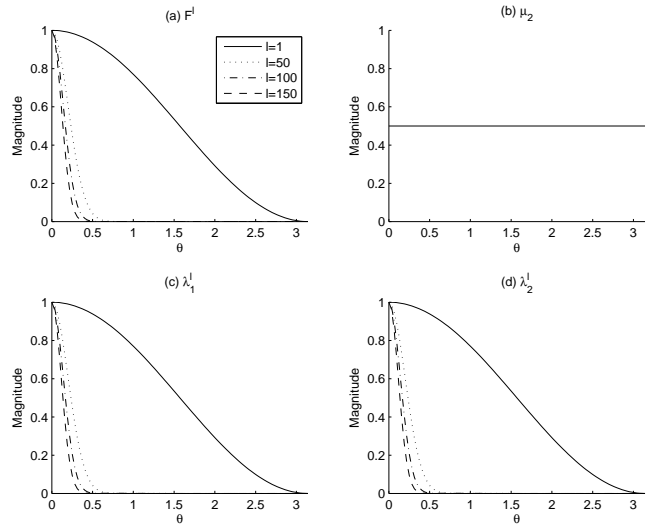}
\caption{Frequency responses of a (Gaussian) matrix of filters with $b=1$, $c=1$, and $d=0$.}
\label{fig:GaussianFilter1}
\end{figure}
\begin{figure}
\centering
\includegraphics[width=4in]{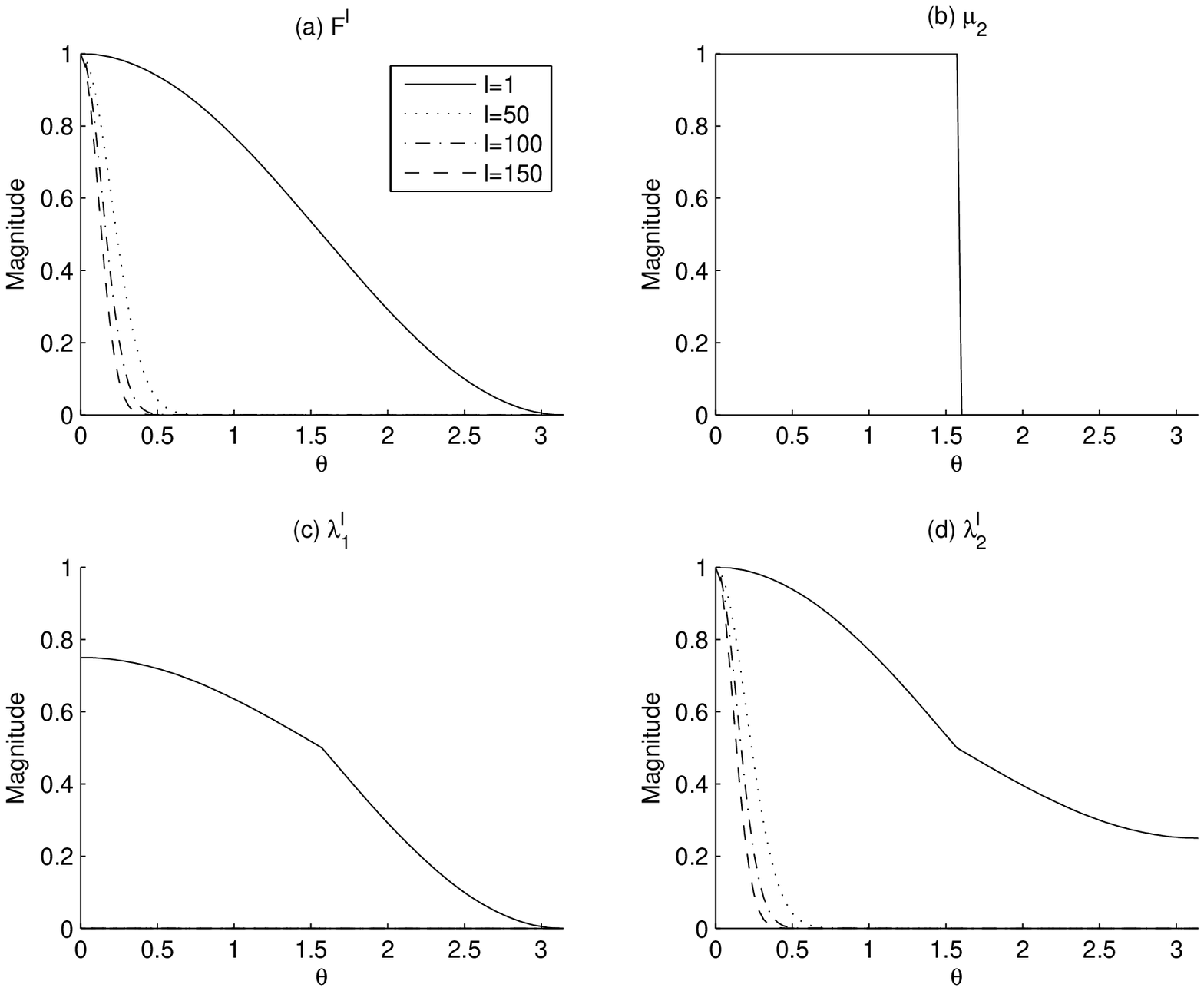}
\caption{Frequency responses of a (Gaussian) matrix of filters with $b=1$, $c=1/2$, and $d=0$.}
\label{fig:GaussianFilter2}
\end{figure}
\begin{figure}
\centering
\includegraphics[width=4in]{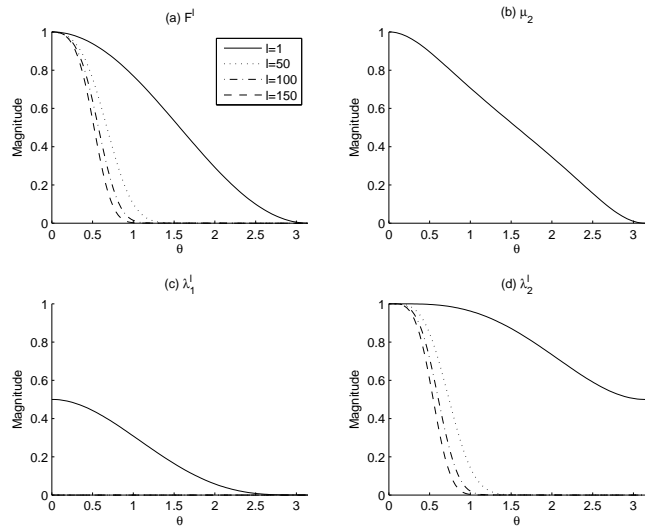}
\caption{Frequency responses of a matrix of filters with $b=1$, $c=0$, and $d=-0.5$.}
\label{fig:NonGaussianFilter1}
\end{figure}
\begin{figure}
\centering
\includegraphics[width=4in]{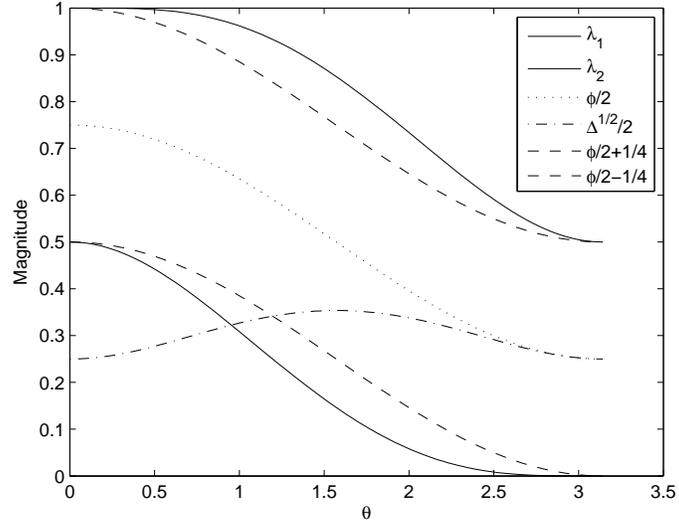}
\caption{Close-up of eigenfunctions and their constituents with $b=1$, $c=0$, and $d=-0.5$.}
\label{fig:EigenCloseup}
\end{figure}
\begin{figure}
\centering
\includegraphics[width=3in]{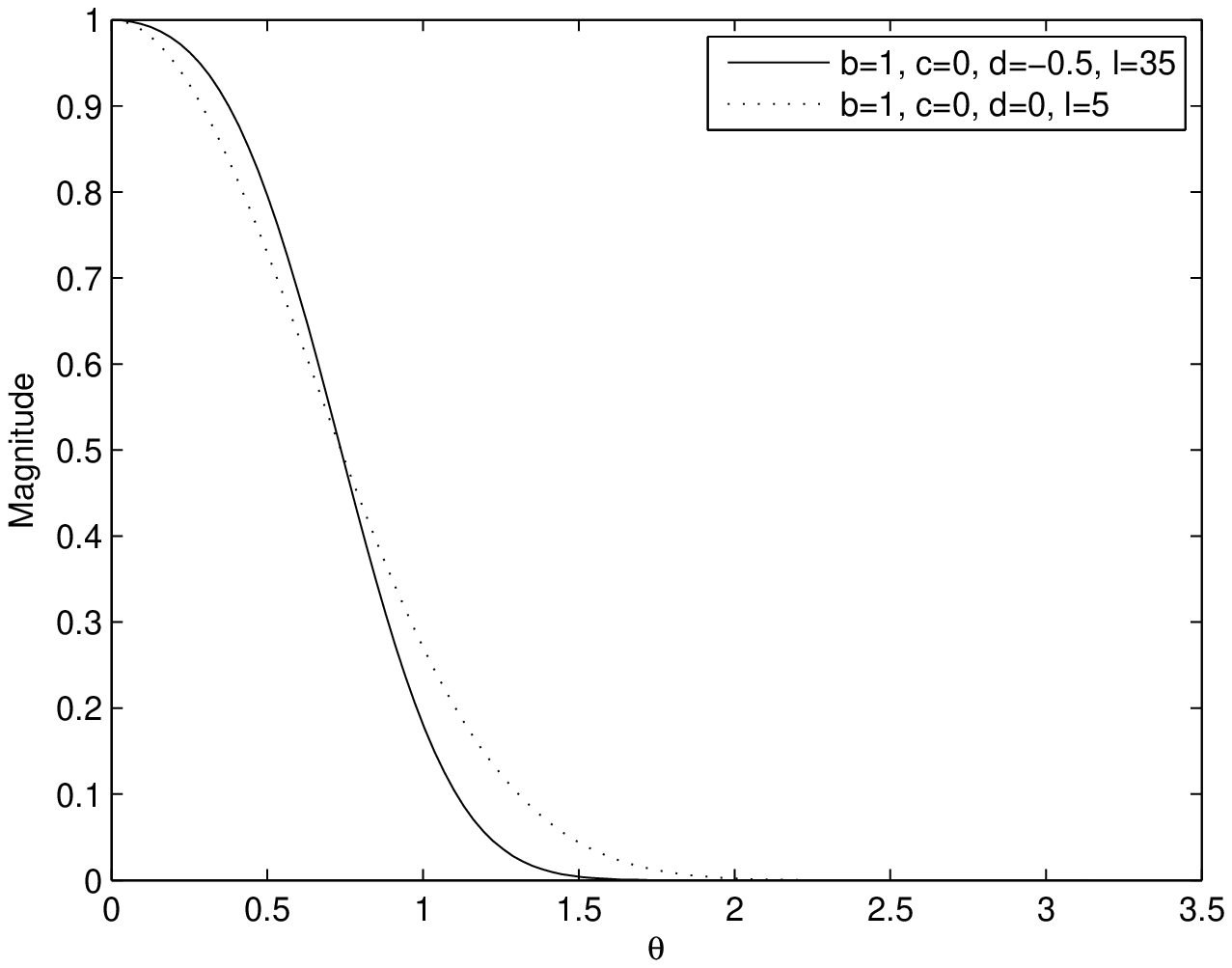}
\caption{Comparison of two filters: Gaussian ($b=1$, $c=d=0$) and Non-Gaussian ($b=1$, $c=0$, $d=-0.5$).}
\label{fig:FilterComparison}
\end{figure}

Another design philosophy to consider is to create a sharp
cut-off in the frequency response profile. Under such approach,
we may want to find a set of parameters that maximizes
$F^l(\pi/16)-F^l(\pi/4)$ at say $l=100$, while satisfying the
constraints of Theorem \ref{theorem2}. This is a non-convex
optimization problem and can be solved numerically with various
software packages.

By a Matlab$\circledR$ Optimization toolbox, we obtained $b=1$,
$c=0.48$, and $d=-0.26$. The result was not sensitive to
initial conditions, which were set randomly. The frequency
response of the filter at $l=100$ is shown in Figure
\ref{fig:SharpCutoffDesign} along with the response of the
Gaussian counterpart. The amount of fall-off for the
non-Gaussian case is 0.93 while that of the Gaussian case is
0.41 and that of a non-Gaussian case with $b=1$, $c=0$ and
$d=-0.5$ (the filter shown in Figure
\ref{fig:NonGaussianFilter1}) is 0.83.

\begin{figure}
\centering
\includegraphics[width=3in]{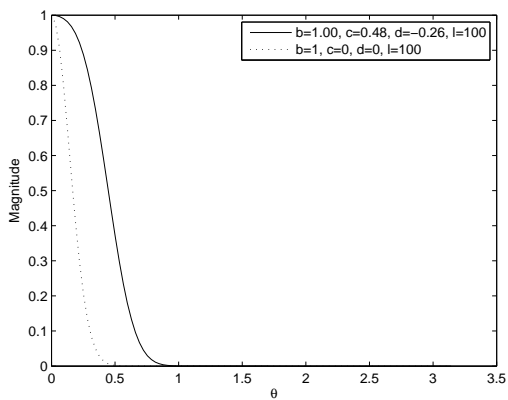}
\caption{Comparison of two filters at $l=100$: Gaussian ($b=1$,
$c=d=0$) and Non-Gaussian ($b=1$, $c=0.48$, $d=-0.26$). The
parameters for the non-Gaussian case are obtained by constrained
non-linear optimization to maximize the fall off between
$\theta=\pi/16$ and $\theta=\pi/4$ (in other words,
$F^l(\pi/16)-F^l(\pi/4)$.). The fall-off for the Gaussian case
is 0.93, and that for the Gaussian case is 0.41.}
\label{fig:SharpCutoffDesign}
\end{figure}

\section{Discussion}
One question that naturally arise from the illustration given
in Figure \ref{fig:EigenCloseup} regarding the frequency
profiles of the constituents is if it is possible to generate
the same frequency response of the matrix of filters by using a
conventional convolution kernel with a larger support (more
than 3 non-zero coefficients). For example, we can use an
equivalent filter at some $l$ as the convolution kernel. The
answer to the question is no, since the formulae of the
frequency responses for the two cases are different; For a
conventional convolution based case,
\begin{equation}\label{eq:freqConv}
F^l(\theta)=\left(K(\theta)\right)^l,
\end{equation}
where $K$ is the frequency response of the convolution kernel.
On the other hand, the frequency response of a matrix of
filters is given by
\begin{equation}\label{eq:freqMF}
F^l(\theta)=\mu_1(\theta)\left(\lambda_1(\theta)\right)^l +
\mu_2(\theta)\left(\lambda_2(\theta)\right)^l.
\end{equation}
Thus, (\ref{eq:freqConv}) cannot generate (\ref{eq:freqMF}) in
general. As seen in Section \ref{experiments}, the converse is
not true, and (\ref{eq:freqMF}) can generate any instance of
(\ref{eq:freqConv}) by setting $\sigma_{xx}(\theta)=K(\theta)$
and $d=0$.

In this paper, we have limited our study to the $2\times 2$
case. Even with such minimal configuration, the resulting
equivalent filter is able to construct a non-trivial (or
non-Gaussian) scale space. With larger configuration, we expect
that more elaborate frequency responses are possible. Note that
the original EQW employed a $3\times 3$ matrix of filters. It
is however, difficult to extend the analysis described in this
paper to the general case.  Closed form expressions of
eigenvalues are not possible for $P>4$, and although they exist
for $P\le 4$, deriving sufficient conditions for scale space
filtering can be extremely complicated.

We can extend the matrix size while imposing some structural
constraints on the matrix. For example, we can consider a
$P\times P$ matrix of filters that are circulant. Then, we will
be able to derive a simple expression for the frequency
response of the equivalent filters. In this case, the mixing
coefficients are all $1/P$, thus the frequency response of the
equivalent filter becomes
\begin{equation}
F^l(\theta)=\frac{1}{P}\sum_{j=0}^P \lambda_j^l(\theta)
\end{equation}
with
\begin{equation}
\lambda_j(\theta)=\sum_{k=0}^P F_{1k}(\theta)\rho_P^{jk}
\end{equation}
where $F_{1k}$ is the frequency response of $\mathbf{f}_{1k}$
and $\rho_P$ is the $P$th root of unity. Note that at $P=2$,
this circulant configuration leads to $\sigma_{xx}=\sigma_{aa}$
and in turn leads a Gaussian scale space. It is not clear if
the same can be said for $P>2$.

Without closed form expressions of eigenvalues, we resort to
numerical schemes. Give a matrix of filter, we want to test if
the filter satisfies the scale space requirements. We need to
come up with numerical conditions that guarantee the positivity
and unimodality requirements at every $\theta$ and the constant
reduction requirement at every $l$.

So far, we assumed that each convolution filter is circulant.
We can extend the results for non-circulant filter with some
type of extension schemes such as zero padding and reflection,
given an upper limit of the iteration number. Let $L$ be the
upper limit of the iteration number. Then, the length of the
equivalent filter is at most $2L+1$. Then the result of the
iterative filtering can be obtained by first extending the
original signal by $L$ on both ends by the chosen extension
scheme, apply the circulant filters to the extended signal, and
truncate the result at the portion of the original signal.
Thus, non-circulant filter can be implemented by circulant one
with proper extension. Given a scale space of a signal (i.e. a
collection of signals that satisfy the scale space
requirements), a truncated portion of the signal also satisfies
the scale space requirements. Thus, the sufficient condition
for the scale space filter remains applicable to the
non-circulant case.

\section{Conclusion}
In this paper, we first derived the frequency response of a
general matrix of filters applied iteratively to the signal.
The response is a convex combination of the power of
eigen-functions describing the impulse response of the filter.
We then studied a $2\times 2$ matrix of filters and derive
sufficient conditions for it to be a scale space kernel. We
showed that the $2\times 2$ matrix of filters can generate
non-Gaussian scale space, thus are more powerful than the
conventional convolution kernel.

Future research goals include extension of the study to more
general matrix sizes. We suggest investigating some general
cases such as circulant one and tri-diagonal one, and derive
sufficient conditions for the scale space requirements. For
more general cases, we suggest deriving a numerical test that
checks if the given configuration satisfies the scale space
requirements.


\begin{thebibliography}{1}
\providecommand{\url}[1]{#1}
\def\UrlFont{\rmfamily}
\providecommand{\newblock}{\relax}
\providecommand{\bibinfo}[2]{#2}
\providecommand\BIBentrySTDinterwordspacing{\spaceskip=0pt\relax}
\providecommand\BIBentryALTinterwordstretchfactor{4}
\providecommand\BIBentryALTinterwordspacing{\spaceskip=\fontdimen2\font plus
\BIBentryALTinterwordstretchfactor\fontdimen3\font minus
  \fontdimen4\font\relax}
\providecommand\BIBforeignlanguage[2]{{%
\expandafter\ifx\csname l@#1\endcsname\relax
\typeout{** WARNING: IEEEtran.bst: No hyphenation pattern has been}%
\typeout{** loaded for the language `#1'. Using the pattern for}%
\typeout{** the default language instead.}%
\else
\language=\csname l@#1\endcsname
\fi
#2}}

\bibitem{Kubota:IJCV2009}
T.~Kubota, ``A shape representation with elastic quadratic
  polynomials--preservation of high curvature points under noisy conditions,''
  \emph{International Journal Computer Vision}, vol.~82, no.~2, pp. 133--155,
  2009.

\bibitem{Witkin:ScaleSpaceA}
A.~Witkin, ``Scale-space filtering,'' in \emph{Proc. 8th Int'l Joint Conf.
  Artificial Intelligence}, 1983, pp. 1019--1022.

\bibitem{Lindeberg:PAMI90}
T.~Lindeberg, ``Scale-space for discrete signals,'' \emph{IEEE Trans. Pattern
  Analysis and Machine Intelligence}, vol.~12, no.~3, pp. 234--254, March 1990.

\bibitem{Koenderink:ScaleSpace}
J.~J. Koenderink, ``The structure of images,'' \emph{Biological Cybernetics},
  vol.~50, pp. 363--370, 1984.

\bibitem{Babaud:ScaleSpace}
J.~Babaud, A.~P. Witkin, M.~Baudin, and R.~O. Duda, ``Uniqueness of the
  gaussian kernel for scale space filtering,'' \emph{IEEE Trans. Pattern
  Analysis and Machine Intelligence}, vol.~8, no.~1, pp. 26--33, January 1986.

\bibitem{Yuille:ScaleSpace}
A.~L. Yuiile and T.~A. Poggio, ``Scaling theorem for zero crossings,''
  \emph{IEEE Trans. Pattern Analysis and Machine Intelligence}, vol.~8, no.~1,
  pp. 15--25, January 1986.

\end{thebibliography}

\end{document}